\documentclass[12pt]{amsart}

\usepackage{epsf,amssymb,amsmath,overpic,amscd,psfrag,stmaryrd,times,mathrsfs, euscript,thmtools, amsbsy, multicol, mathtools}

\usepackage{caption}
\usepackage{subcaption}
\usepackage[all]{xy}
\usepackage{hyperref, stmaryrd}
\usepackage[dvipsnames]{xcolor}
\usepackage{enumitem, bbm}
\usepackage{tikz,tikz-cd}
\usepackage{array}
\usepackage{setspace}
\usepackage{diagbox}

\tikzcdset{scale cd/.style={every label/.append style={scale=#1},
    cells={nodes={scale=#1}}}}

\usepackage{todonotes}
\usetikzlibrary{decorations.markings}

\newtheorem{theorem}{Theorem}[section]
\newtheorem{lemma}[theorem]{Lemma}

\theoremstyle{definition}
\newtheorem{definition}[theorem]{Definition}
\newtheorem{example}[theorem]{Example}

\newtheorem{remark}[theorem]{Remark}

\numberwithin{equation}{subsection}

\DeclareMathAlphabet{\mathpgoth}{OT1}{pgoth}{m}{n}

\DeclareMathAlphabet{\mathpzc}{OT1}{pzc}{m}{it}
\newcommand{\R}{\mathbb{R}}
\newcommand{\Z}{\mathbb{Z}}

\newcommand{\be}{\begin{enumerate}}
\newcommand{\ee}{\end{enumerate}}
\newcommand{\op}{\operatorname}

\newcommand{\iT}[1]{#1}

\newcommand{\chartData}[1]{\mathcal{C}_{\iT{I}}}

\newcommand{\map}{\op{Map}}

\newcommand{\hF}{\widehat{F}}

\newcommand{\hX}{\widehat{X}}
\newcommand{\hf}{\widehat{f}}

\makeatletter
\newcommand{\labitem}[2]{%
\def\@itemlabel{#1}
\item
\def\@currentlabel{#1}\label{#2}}
\makeatother

\begin{document}

\title[Equivariant Neural Networks and Equivarification]
{Equivariant Neural Networks and Equivarification}

\author{Erkao Bao}
\address{School of Mathematics, University of Minnesota, Minneapolis, MN 55455}
\author{Jingcheng Lu}
\author{Linqi Song}
\author{Nathan Hart-Hodgson}
\author{William Parson}
\author{Yanheng Zhou}

\keywords{equivariant, equivarification, group, neural network, group action, universal property, group convolution, symmetry}

\thanks{This project is partially supported by NSF Grants DMS-2404529.}

\begin{abstract}
Equivariant neural networks are a class of neural networks designed to preserve symmetries inherent in the data. In this paper, we introduce a general method for modifying a neural network to enforce equivariance, a process we refer to as \emph{equivarification}. We further show that group convolutional neural networks (G-CNNs) arise as a special case of our framework.
\end{abstract}

\maketitle

\setcounter{tocdepth}{2}

%  \tableofcontents

\section{Introduction} \label{sec: introduction}

Equivariant neural networks are a class of architectures designed to respect symmetries inherent in the data. By symmetry, we mean the action of a group (see Definition~\ref{def:group action}) on the input space, such as rotations or reflections of images. Originating from the works on Group Convolutional Neural Networks (G-CNNs)~\cite{cohen2016group,cohen2018spherical} and steerable neural networks~\cite{cohen2017steerable,weiler20183d}, and supported by a rich theoretical foundation~\cite{cohen2019general,esteves2020theoretical,lim2023equivariant,gibson2024equivariant}, equivariant architectures have seen rapid development across a broad range of applications, including image classifications \cite{bernander2021replacing,bernander2022rotation,pandya20232,qiao2023scale}, physics modeling \cite{bogatskiy2020lorentz,batzner20223,favoni2022lattice, gong2023general,yuan2024equivariant,tang2024deep}, chemical reaction simulations \cite{van2023equireact}, pharmaceutical prediction \cite{cremer2023equivariant}, robotic control \cite{seo2025se}, and natural language processing \cite{hutchinson2021lietransformer,white2022equivariant}. To bridge the gap with large-scale industrial models, recent work has also focused on fine-tuning pretrained networks with approximate equivariance~\cite{basu2023equi,mondal2023equivariant}.

In this work, we focus on equivariant neural networks for classification tasks, where a natural requirement is \emph{invariance} under data symmetries — for instance, an image and its rotated versions should be classified identically. This motivates the use of \emph{equivariant} models, whose outputs transform predictably under group actions. For example, an equivariant image classifier may simultaneously output both the predicted object class and the orientation of the image (e.g., $0^\circ$, $90^\circ$, $180^\circ$, $270^\circ$). A $180^\circ$ rotation of the input should be reflected by a corresponding update in the rotation prediction (e.g., from $90^\circ$ to $270^\circ$), while the object classification remains invariant.

Conventional classifiers treat transformed inputs as unrelated, relying on data augmentation during training to enforce output symmetry. This yields only implicit, approximate equivariance. In contrast, equivariant networks embed the symmetry analytically, making the resulting models not only conceptually principled but also computationally more efficient.

Despite their strong theoretical appeal, both G-CNNs and steerable networks face practical limitations due to their case-by-case constructions. G-CNNs require an explicit redefinition of the convolutional layers on the product space, $X \times G$, of the input and symmetry group, significantly increasing the computational cost and complicating integration with standard backbones. Steerable networks enhance representation capability by incorporating nontrivial group representations to model vector or tensor features, enabling the group to act nontrivially on the feature space. However, this generality entails complexity: kernel parameterizations must be carefully derived for each irreducible representation of the group, leading to intricate, group-specific implementations that are hard to extend to new groups or architectures.

To enable a convenient and modular construction of equivariant classifiers, we propose a general framework called \emph{equivarification}, which transforms \emph{any} feedforward neural network into an equivariant one with respect to a prescribed group action on the input. Specifically, our multi-view-based equivarification offers the following features:
\begin{itemize}
    \item It is fully architecture-agnostic and requires no access to or modification of internal network layers.
    \item When applied globally, the equivarified model retains the \emph{exact same number of parameters} as the original model via parameter sharing.
    \item The network can jointly perform classification and transformation prediction (e.g., identifying the rotation applied to the input image).
\end{itemize}

The rest of the paper is organized as follows. In Section~\ref{sec:prelim of group action}, we introduce the mathematical preliminaries on group actions and equivariant functions. Section~\ref{sec: equivarification} develops the theoretical framework underlying our universal equivarification procedure. Section~\ref{sec: Application to neural networks} provides practical constructions for finite symmetry groups. In particular, we show that the layer-wise equivarification can be viewed as a kernel-sharing variant of G-CNNs. The experimental results are presented in Section~\ref{sec:experiments}, where we compare our method with standard data enhancement approaches.

\section{Preliminaries}\label{sec:prelim of group action}

We now define the key notions of group actions and equivariance.

Let $G$ be a group and $X$ a set.

\begin{definition}[Group Action]\label{def:group action}
A (left) action of $G$ on $X$ is a map
\[
T: G \times X \to X
\]
satisfying:
\begin{itemize}
    \item $T(e, x) = x$ for all $x \in X$, where $e \in G$ is the identity element;
    \item $T(g_1, T(g_2, x)) = T(g_1g_2, x)$ for all $g_1, g_2 \in G$ and $x \in X$.
\end{itemize}
When the action is clear from context, we often omit $T$ and write $gx$ in place of $T(g, x)$. That is, the second condition becomes $g_1(g_2 x) = (g_1g_2)x$.

We say the action is \emph{trivial} if $gx = x$ for all $g \in G$ and $x \in X$.
\end{definition}

Now suppose $G$ acts on two sets $X$ and $Y$.

\begin{definition}[Equivariance and Invariance]
A map $F: X \to Y$ is said to be \emph{$G$-equivariant} if
\[
F(gx) = gF(x)
\quad \text{for all } g \in G,\ x \in X.
\]
If $G$ acts trivially on $Y$, then we say that $F$ is \emph{$G$-invariant}.
\end{definition}

\begin{example}[Equivariance in Image Classification]\label{example: equivarification}
Let $X$ be the input space of all $28 \times 28$ grayscale images in the MNIST dataset, representing handwritten digits from $0$ to $9$. Let $G = \Z/4\Z$ be the cyclic group of order 4, generated by a rotation $g$ of $90^\circ$ counterclockwise. We define the action of $G$ on $X$ by letting $gx$ be the image obtained by rotating $x$ by $90^\circ$.

We set the output space $Y = \{0,1,\ldots,9\} \times \{0, 90, 180, 270\}$, where each output contains both a digit class and an estimated rotation angle. Define the group action on $Y$ by
\[
g(\text{num}, \theta) = (\text{num},\ \theta + 90 \mod 360).
\]
An equivariant neural network $F: X \to Y$ satisfies
\[
F(gx) = gF(x),
\]
meaning that rotating the input image by $90^\circ$ updates the predicted angle accordingly, without changing the predicted digit.
\end{example}

This framework allows us to model each layer in a neural network as a $G$-set, where the group $G$ encodes symmetry transformations and acts on the space $X$ of inputs to that layer. The action may represent concrete transformations of the input (e.g., rotations or reflections), or more abstract symmetries emerging in intermediate layers.

By enforcing equivariance layer by layer, we can systematically build neural networks that respect the symmetries of the task. The same group can be used to describe different types of operations across layers—rotation, flipping, or others—enabling consistent and efficient architecture design.

\section{Equivarification}\label{sec: equivarification}

In this section, we present the core theoretical framework underlying our equivarification method, which allows any map to be lifted to an equivariant one. Readers primarily interested in applications may wish to skip to the examples, which offer intuitive illustrations of how equivariant neural networks are constructed.

Throughout this section, we fix a group $G$ acting on a set $X$, and let $Z$ be any target set.

\begin{definition}\label{main def}
The \emph{$G$-product} of $Z$ is the set of functions from $G$ to $Z$:
\[
\map(G, Z) := \{ s: G \to Z \}.
\]
\end{definition}

We define a natural (left) $G$-action on $\map(G, Z)$ by
\begin{align*}
G \times \map(G, Z) &\to \map(G, Z) \\
(g, s) &\mapsto gs, \quad \text{where } (gs)(g') := s(g^{-1}g').
\end{align*}

There is a canonical projection map
\begin{align}\label{projection}
p: \map(G, Z) \to Z, \quad s \mapsto s(e),
\end{align}
which evaluates a function at the identity element $e \in G$.

\begin{lemma}\label{trivial lemma}
Given any map $F: X \to Z$, there exists a unique $G$-equivariant map $\widehat{F}: X \to \map(G, Z)$ such that
\[
p(\widehat{F}(x)) = F(x), \quad \text{for all } x \in X.
\]
\end{lemma}

\begin{proof}
Define $\widehat{F}(x) \in \map(G, Z)$ by
\[
(\widehat{F}(x))(g) := F(g^{-1}x), \quad \text{for all } g \in G.
\]
To verify equivariance, let $g, h \in G$ and $x \in X$. Then:
\[
(g \cdot \widehat{F}(x))(h) = \widehat{F}(x)(g^{-1}h) = F(h^{-1}g x) = (\widehat{F}(g x))(h).
\]
So $\widehat{F}(g x) = g \cdot \widehat{F}(x)$, as required. Uniqueness follows from the requirement that the lift satisfies $p \circ \widehat{F} = F$ and equivariance.
\end{proof}

\begin{remark}
In Definition~\ref{main def} and Lemma~\ref{trivial lemma}, both $G$ and $Z$ are arbitrary sets. This construction can be generalized to other categories. For example, if $G$ is a compact Lie group and $Z$ is a smooth manifold, one may restrict to smooth maps $\map(G, Z)$. If $G$ is non-compact, compactly supported smooth maps may be more appropriate. 

In practice, especially for neural network implementation, infinite groups must be approximated by finite subsets. Thus, for the remainder of the paper, we assume that $G$ is finite to avoid such technical complications.
\end{remark}

The above lemma is captured in the following commutative diagram:

\begin{figure}[h]
\centering
\begin{tikzpicture}
  \matrix (m) [matrix of math nodes,row sep=3em,column sep=4em,minimum width=2em]
  {
     & \map(G, Z) \\
     X & Z \\};
  \path[->]
    (m-2-1) edge node [below] {$F$} (m-2-2)
    (m-1-2) edge node [right] {$p$} (m-2-2)
    (m-2-1) edge [dashed,->] node [above] {$\exists!\, \widehat{F}$} (m-1-2);
\end{tikzpicture}
\caption{Every map $F: X \to Z$ admits a unique $G$-equivariant lift $\widehat{F}$ such that $p \circ \widehat{F} = F$.}
\label{commutative diagram}
\end{figure}

This motivates the general notion of $G$-equivarification:

\begin{definition}[$G$-Equivarification]\label{lemma: equivarification}
A triple $(\widehat{Z}, T, p)$ is called a \emph{$G$-equivarification} of $Z$ if:
\begin{itemize}
    \item $\widehat{Z}$ is a set equipped with a $G$-action $T$;
    \item $p: \widehat{Z} \to Z$ is a map;
    \item For any $G$-set $X$ and any map $F: X \to Z$, there exists a $G$-equivariant map $\widehat{F}: X \to \widehat{Z}$ such that $p \circ \widehat{F} = F$.
\end{itemize}
We typically omit $T$ from the notation and refer to $\widehat{Z}$ and $p$ as the equivarification of $Z$.
\end{definition}

In Section~\ref{sec: Application to neural networks}, we show how this abstract notion applies to neural networks. Lemma~\ref{trivial lemma} demonstrates that the triple $(\map(G, Z), \text{action}, p)$ forms a canonical equivarification. Other equivarifications exist and are discussed further in the Appendix.

\begin{example}\label{ex:C4}
Let $G = \Z/4\Z = \{e, g, g^2, g^3\}$, the cyclic group of order 4. Then $s\in\map(G, Z)$ can be identified with $Z^4$ via the map:
\begin{equation}\label{eq:C4 tuple}
s \mapsto (s(e), s(g), s(g^2), s(g^3)).
\end{equation}
Under this identification, $G$ acts on $Z^4$ by cyclic permutation:
\[
g \cdot (z_0, z_1, z_2, z_3) = (z_3, z_0, z_1, z_2).
\]
Recall that the projection $p: \map(G,Z) \to Z$ evaluates $s$ at the identity, i.e., $p(s) = s(e)$. Under the tuple representation \eqref{eq:C4 tuple}, this corresponds to projecting to the first coordinate:
\[
p(z_0, z_1, z_2, z_3) = z_0.
\]

Let $F: X \to Z$ be any function. Then the lifted map $\widehat{F}: X \to Z^4$ is given by
\[
\widehat{F}(x) = \left(F(x), F(g^{-1}x), F(g^{-2}x), F(g^{-3}x)\right),
\]
which is easily verified to be $G$-equivariant. Furthermore, $p \circ \widehat{F} = F$.
\end{example}

\section{Application of equivarification to neural networks} \label{sec: Application to neural networks}

In this section, we illustrate how our proposed equivarification method can be applied to construct equivariant neural networks and explain its connection to group convolution.

\subsection{A multi-view construction}
While Lemma \ref{trivial lemma} ensures the existence and uniqueness of an equivariant lift (up to isomorphism), the concrete realization of such lifts in neural network architectures can vary widely. As illustrated in Example~\ref{ex:C4}, any lifted map $s \in \map(G, Z)$ can be identified with its values on the group elements.

As a natural extension, we propose the following multi-view-based equivarification, which constructs an equivariant map by evaluating the original function on group-transformed inputs: 
\begin{equation}\label{eq:equinet global}
    \widehat{F}(x) = \left(F(g_1 x), F(g_2 x), \cdots, F(g_{|G|} x)\right),
\end{equation}
where $|G|$ is the cardinality of $G$, and here we choose an ordering $(g_1, g_2, \dots, g_{|G|})$ of elements of $G$.

We refer to this construction as \emph{global equivarification}. It provides a convenient and architecture-agnostic method to convert any map \( F \in \map(X, Z) \) into an equivariant map \( \widehat{F}: X \to Z^{|G|} \) satisfying
\[
\widehat{F}(g x) = g \cdot \widehat{F}(x), \quad \text{for all } g \in G.
\]
Importantly, the equivarified network retains the same number of trainable parameters as the original model, since all evaluations $F(g_i x)$ share the same set of weights. This weight sharing across transformed inputs ensures parameter efficiency while maintaining equivariance.

Here, \( \widehat{F}(x) \in Z^{|G|} \) is interpreted as a tuple indexed by the group elements. The group action on tuples \( a = (a_{g_1}, \dots, a_{g_{|G|}}) \in Z^{|G|} \) is defined by right multiplication on the indices:
\[
g \cdot (a_{g_1}, a_{g_2}, \dots, a_{g_{|G|}}) = (a_{g_1 g}, a_{g_2 g}, \dots, a_{g_{|G|} g}).
\]

When the symmetry group \( G \) is large, the dimension of \( \widehat{F}(x) \in Z^{|G|} \) may become prohibitively high. In such cases, one can seek a more compact representation via \emph{$\phi$-equivariant} neural networks. We introduce the following notions.

Let $G$ and $H$ be groups, let $\phi: G \to H$ be a group homomorphism, and suppose that $G$ acts on a space $X$ while $H$ acts on a space $Y$.

\begin{definition}
A map $F: X \to Y$ is called \emph{$\phi$-equivariant} if 
\[
\phi(\vartheta)\,F(x) \;=\; F(\vartheta x)
\qquad \text{for all }\vartheta \in G,\; x \in X.
\]
\end{definition}

\begin{lemma}
Let $\phi: G \to H$ be a surjective homomorphism.  
Let $X$ and $Z$ be spaces, and suppose $G$ acts on $X$.  
Given a map $F: X \to Z$, there exists a $\phi$-equivariant lift 
\[
\widehat{F}: X \to \op{Map}(H, Z),
\qquad
p\circ \widehat{F} = F,
\]
where $p: \op{Map}(H,Z)\to Z$ denotes the projection map, if and only if $F$ is $(\ker\phi)$-invariant, i.e.
\[
F(\vartheta x) = F(x)
\qquad
\text{for all }\vartheta \in \ker\phi,\; x \in X.
\]
Moreover, such a lift is unique.
\end{lemma}

Note that $\ker\phi$ is a normal subgroup%
\footnote{A subgroup $N \subseteq G$ is called \emph{normal} if it is invariant under conjugation by elements of $G$, i.e.
\[
gng^{-1} \in N \quad \text{for all }g\in G,\; n\in N.
\]
Equivalently, $N$ is normal if the left and right cosets coincide: $gN=Ng$ for all $g\in G$.} 
and that $G/\ker\phi \simeq H$.%
\footnote{If $N\subseteq G$ is normal, the \emph{quotient group} (or factor group) $G/N$ consists of the left cosets of $N$:
\[
G/N := \{gN \mid g\in G\}.
\]
Multiplication is defined by
\[
(gN)(hN) := (gh)N \qquad \forall g,h\in G,
\]
which is well-defined precisely because $N$ is normal.}

The proof is a straightforward modification of Lemma~\ref{trivial lemma} and is omitted.

\medskip

In practice, starting with a group $G$, one may wish to reduce the group size layer by layer.  
A common approach is to choose a normal subgroup $N\subseteq G$.  
Given a map $F: X \to Z$ that is not $N$-invariant, we can modify it to an $N$-invariant map $F_{\text{inv}}: X \to Z$.  
Since under the $N$-action the space $X$ decomposes as a disjoint union of $N$-orbits,
\[
X \;=\; O_1 \,\coprod\, \dots \,\coprod\, O_k,
\qquad
O_i := \{\varphi x_i \mid \varphi\in N\}\;\text{ for some }x_i\in X,
\]
we need $F_{\text{inv}}$ to be constant on each $O_i$.  
For instance, if $Y$ is a vector space, we may define:

\begin{enumerate}
\item When $Y=\R$, take
\[
F_{\text{inv}}(x) \;=\; \max_{\varphi\in N}\,\bigl|F(\varphi x)\bigr|.
\]

\item Averaging over $N$:
\[
F_{\text{inv}}(x) \;=\; \frac{1}{|N|}\sum_{\varphi\in N}F(\varphi x).
\]

\item More generally, for any function $\rho: G/N \times Y \to Y$,
\[
F_{\text{inv}}(x) \;=\; \sum_{\varphi\in N}\rho\bigl(\pi(\varphi),F(\varphi x)\bigr),
\]
where $\pi: G\to G/N$ is the canonical projection.
\end{enumerate}

\begin{example}[$\phi$-equivariant network on $D_4$]\label{ex: D4 phi-equi}
Let $D_4$ be the dihedral group of order 8, which admits the explicit representation
\[
D_4 = \{e, r, r^2, r^3, s, sr, sr^2, sr^3\},
\]
where, when acting on images, $r$ denotes $90^\circ$ rotation, and $s$ denotes horizontal flip. These elements satisfy the relations $r^4 = e$, $s^2 = e$, and $sr = r^{-1}s$.

The subgroup $C_4 = \{e, r, r^2, r^3\}$ (i.e., $90^\circ-$rotation group) is a normal subgroup 
of $D_4$, and the quotient group $D_4 / C_4$
is isomorphic to $\Z_2 = \{e, s\}$. Define a homomorphism $\phi: D_4 \to \Z_2$ by
\[
\phi(s^m r^n) = s^m, \quad \text{for } m,n \in \mathbb{N}.
\]
Then a $\phi$-equivariant lift of a function $F$ can be constructed by first forming an $C_4$-invariant mapping, such as:
\[
F_{\text{inv}}(x) = \frac{1}{4}\sum_{g \in C_4} F(g x) \quad \text{or} \quad F_{\text{inv}}(x) = \max_{g \in C_4} F(g x),
\]
and then defining the lifted function:
\[
\hF_\phi(x) := \left(F_{\text{inv}}(x),\ F_{\text{inv}}(s x)\right).
\]

Now, for any $g = s^m r^n \in D_4$, we compute:
\[
\begin{aligned}
\hF_\phi(gx)
&= \left(F_{\text{inv}}(s^m r^n x),\ F_{\text{inv}}(s^m r^n s x)\right) \\
&= \left(F_{\text{inv}}(r^{n(-1)^m} s^m x),\ F_{\text{inv}}(r^{n(-1)^m} s^{m+1} x)\right) \\
&= \left(F_{\text{inv}}(s^m x),\ F_{\text{inv}}(s^{m+1} x)\right) \\
&= s^m \cdot \left(F_{\text{inv}}(x),\ F_{\text{inv}}(s x)\right) \\
&= \phi(g)\cdot \hF_\phi(x),
\end{aligned}
\]
where the $\Z_2$ action on a 2-tuple is defined as the flip:
\[
s \cdot (a_0, a_1) = (a_1, a_0), \quad e \cdot (a_0, a_1) = (a_0, a_1).
\]
This verifies that $\hF_\phi$ is $\phi$-equivariant with respect to the group homomorphism $\phi : D_4 \to \Z_2$.
\end{example}

\subsection{Layer-by-layer equivarification}
The global lifting approach \eqref{eq:equinet global} naturally extends to a layer-wise equivarification framework, as illustrated in Figure~\ref{fig:layer by layer}. We focus here on convolutional neural networks for image data. The input layer $X_0 = \widehat{X}_0$ consists of raw input images. For each $i \geq 1$, a feature $\widehat{x}_i \in \widehat{X}_i$ is represented as a $|G|$-tuple of the form:
\[
\widehat{x}_i = \big(\widehat{F}_{i-1}(g_1 \widehat{x}),\, \widehat{F}_{i-1}(g_2 \widehat{x}),\, \cdots,\, \widehat{F}_{i-1}(g_{|G|} \widehat{x})\big), \quad \widehat{x} \in \widehat{X}_{i-1}.
\]
At the input layer, the $G$-action corresponds to physical transformations such as rotations or flips. For subsequent layers $i \geq 1$, the $G$-action is implemented as index permutation over $|G|$-tuples.

Suppose each $x_i \in X_i$ has shape $\mathbb{R}^{C_i \times H_i \times W_i}$, where $C_i$ is the number of channels, and $H_i$, $W_i$ are spatial dimensions. To construct the next layer, we concatenate all components of the tuple $\widehat{x}_i \in \widehat{X}_i$ along the channel dimension and apply a map:
\[
F_i: \mathbb{R}^{|G| \cdot C_i \times H_i \times W_i} \to \mathbb{R}^{C_{i+1} \times H_{i+1} \times W_{i+1}}.
\]
This map is then lifted to an equivariant version:
\begin{equation}\label{eq:equinet layer by layer}
\widehat{F}_i(\widehat{x}_i) = \big(F_i(g_1 \widehat{x}_i),\, F_i(g_2 \widehat{x}_i),\, \cdots,\, F_i(g_{|G|} \widehat{x}_i)\big).
\end{equation}
With this construction, each layer satisfies the desired equivariance:
\[
\widehat{F}_i(g \widehat{x}) = g \cdot \widehat{F}_i(\widehat{x}), \quad \forall g \in G.
\]
Importantly, due to shared weights across different group-transformed branches, this approach introduces no increase in model parameters. For further demonstration, in Figure \ref{fig: layer by layer equinet} we also present a concrete structure of three-layer $C_4-$equivarification that applies to MNIST dataset, including greyscale images of size $28\times 28$ for handwritten digits $0\sim 9$.

\begin{figure}[h]
\centering
\begin{tikzcd}
    X_0 = \widehat{X}_0 \arrow[r, "\widehat{F}_0"] \arrow[dr, "F_0"]  &  \widehat{X}_1 \arrow[r, "\widehat{F}_1"] \arrow[d, "p_1"] \arrow[dr, "F_1"]  &  \cdots \arrow[r, "\widehat{F}_{i-1}"] & \widehat{X}_i \arrow[r, "\widehat{F}_i"] \arrow[dr, "F_i"]&  \widehat{X}_{i+1}\arrow[d, "p_{i+1}"] \arrow[r, "\widehat{F}_{i+1}"] \arrow[dr, "F_{i+1}"] & \cdots \arrow[r, "\widehat{F}_{n-1}"]  & \widehat{X}_n \arrow[r, "\widehat{F}_n"] \arrow[dr, "F_n"] & \widehat{X}_{n+1} \arrow[d, "p_n"] \\
    & X_1 &  \dots & & X_{i+1} & \cdots & & X_{n+1} 
\end{tikzcd}
\caption{Structure of layer-by-layer equivarification. Each layer performs both a standard mapping $F_i$ and its lifted equivariant version $\widehat{F}_i$, followed by projection $p_i: X_i \to \widehat{X}_i$.}
\label{fig:layer by layer}
\end{figure}

\begin{figure}
\centering
\begin{tikzpicture}[scale = 1.1, every node/.style={scale=0.9}]
\draw (0,0) rectangle ++ (1,1) (0.5, 0.5) node {$X_0$}; 
\draw (3,0) rectangle ++ (1,1) (3.5, 0.5) node {$X_1$} (3.5, 1.25) node {$\times$};
\draw (3,1.5) rectangle ++ (1,1) (3.5,2) node {$X_1$} (3.5, 2.75) node {$\times$};
\draw (3,3) rectangle ++ (1,1) (3.5,3.5) node {$X_1$} (3.5, 4.25) node {$\times$};
\draw (3,4.5) rectangle ++ (1,1) (3.5, 5) node {$X_1$} (3.5, 6) node {$\widehat X_1$};
\draw[blue!50] (2.8, -0.2) rectangle ++ (1.4,6);
\draw[->] (1,0.5) -- (3,0.5) node[pos = 0.5]{conv1};
\draw[->] (1,0.5) -- (3,2) node[pos = 0.5]{conv1$\circ g^{-1}$};;
\draw[->] (1,0.5) -- (3,3.5) node[pos = 0.5]{conv1$\circ g^{-2}$};
\draw[->] (1,0.5) -- (3,5) node[pos = 0.5]{conv1$\circ g^{-3}$};

\draw (3 + 3.5, 0 + 2) rectangle ++ (1,1) (3.5+ 3.5, 0.5+ 2) node {$X_2$} (3.5+ 3.5, 1.25+ 2) node {$\times$};
\draw (3+ 3.5,1.5+ 2) rectangle ++ (1,1) (3.5+ 3.5,2+ 2) node {$X_2$} (3.5+ 3.5, 2.75+ 2) node {$\times$};
\draw (3+ 3.5,3+ 2) rectangle ++ (1,1) (3.5+ 3.5,3.5+ 2) node {$X_2$} (3.5+ 3.5, 4.25+ 2) node {$\times$};
\draw (3+ 3.5,4.5+ 2) rectangle ++ (1,1) (3.5+ 3.5, 5+ 2) node {$X_2$} (3.5+ 3.5, 6+ 2) node {$\widehat X_2$};
\draw[blue!50] (2.8+ 3.5, -0.2+ 2) rectangle ++ (1.4,6);
\draw[->] (1+ 3.25,0.5+ 2) -- (3+ 3.5,0.5+ 2) node[pos = 0.5]{conv2};
\draw[->] (1+ 3.25,0.5+ 2) -- (3+ 3.5,2+ 2) node[pos = 0.5]{conv2$\circ g^{-1}$};;
\draw[->] (1+ 3.25,0.5+ 2) -- (3+ 3.5,3.5+ 2) node[pos = 0.5]{conv2$\circ g^{-2}$};
\draw[->] (1+ 3.25,0.5+ 2) -- (3+ 3.5,5+ 2) node[pos = 0.5]{conv2$\circ g^{-3}$};

\draw (3 + 3.5 + 3, 0 + 2 + 2) rectangle ++ (1,1)  (3.5+ 3.5+3, 0.5+ 2+2) node {$X_3$};
\draw[->] (1+ 3.5+3.25,0.5+ 2+2) -- (3+ 3.5+3,0.5+ 2+2) node[pos = 0.5]{conv3};
\draw (3 + 3.5 + 3+1+1.5, 0 + 2 + 2) rectangle ++ (1,1)  (3.5+ 3.5+3+1.5+1, 0.5+ 2+2) node {$X_4$};
\draw[->] (1+ 3.5+3.25+2.75,0.5+ 2+2) -- (3+ 3.5+3+2.5,0.5+ 2+2) node[pos = 0.5]{pool};

\draw (3 + 3.5 + 3+1+1.5 + 2.5, 0 + 2 + 2) rectangle ++ (1,1)  (3.5+ 3.5+3+1.5+1+ 2.5, 0.5+ 2+2) node {$X_5$}  (3.5+ 3.5+ 3+1+1.5 + 2.5, 1.25+ 2+2) node {$\times$};
\draw (3+ 3.5+ 3+1+1.5 + 2.5,1.5+ 2+ 2) rectangle ++ (1,1) (3.5+ 3.5+3+1.5+1+ 2.5,2+ 2+ 2) node {$X_5$} (3.5+ 3.5+ 3+1+1.5 + 2.5, 2.75+ 2+ 2) node {$\times$};
\draw (3+ 3.5+ 3+1+1.5 + 2.5,3+ 2+ 2) rectangle ++ (1,1) (3.5+ 3.5+3+1.5+1+ 2.5,3.5+ 2+ 2) node {$X_5$} (3.5+ 3.5+ 3+1+1.5 + 2.5, 4.25+ 2+ 2) node {$\times$};
\draw (3+ 3.5+ 3+1+1.5 + 2.5,4.5+ 2+ 2) rectangle ++ (1,1) (3.5+ 3.5+3+1.5+1+ 2.5, 5+ 2+ 2) node {$X_5$} (3.5+ 3.5+ 3+1+1.5 + 2.5, 6+ 2+ 2) node {$\widehat X_5 = \R^{40}$};
\draw[blue!50] (2.8+ 3.5+ 3+1+1.5 + 2.5, -0.2+ 2+ 2) rectangle ++ (1.4,6);

\draw[->] (1+ 3.5+3.25+2.75+ 2.5,0.5+ 2+2) -- (3+ 3.5+3+2.5+ 2.5,0.5+ 2+2) node[pos = 0.5]{dense};
\draw[->] (1+ 3.5+3.25,0.5+ 2+2) -- (3+ 3.5+3+2.5+ 2.5,0.5+ 2+2 + 1.5) node[pos = 0.5]{dense $\circ$ pool $\circ$ conv3 $\circ g^{-1}$};
\draw[->] (1+ 3.5+3.25,0.5+ 2+2) -- (3+ 3.5+3+2.5+ 2.5,0.5+ 2+2 + 3) node[pos = 0.5]{dense $\circ$ pool $\circ$ conv3 $\circ g^{-2}$};
\draw[->] (1+ 3.5+3.25,0.5+ 2+2) -- (3+ 3.5+3+2.5+ 2.5,0.5+ 2+2 + 4.5) node[pos = 0.5]{dense $\circ$ pool $\circ$ conv3 $\circ g^{-3}$};

\end{tikzpicture}
\caption{An equivariant CNN structure with three convolution layers on MNIST. Each layer is lifted via multi-view evaluation over group-transformed inputs. The final classification is obtained via argmax over $\hX_5 = \R^{40}$, which encodes 10 classes $\times$ 4 group elements. }
\label{fig: layer by layer equinet}
\end{figure}

\subsection{G-CNNs as a special case of equivarification}

We now compare our construction with that of G-CNNs. 
To do so, we first slightly generalize and reformulate the usual definition of a G-CNN.

Let $X$ and $Y$ be two spaces and assume, in particular, that $Y$ is a vector space (in fact it suffices for $Y$ to be an abelian group). 
A map 
\[
\kappa \in \op{Map}(G \times X, Y)
\]
is called a \emph{convolution kernel}.  
Given such a kernel, we define a $G$-equivariant map
\begin{equation}\label{eqn: GCNN reformulated}
\Phi : \op{Map}(G, X) \longrightarrow \op{Map}(G, Y),
\qquad
\Phi(f)(\varphi) = 
\sum_{\vartheta \in G} 
\kappa\bigl(\varphi^{-1}\vartheta,\, f(\vartheta)\bigr)
\quad \text{for all }\varphi \in G.
\end{equation}
This operation is the usual \emph{group convolution}.  
Here $\op{Map}(G, X)$ represents the $i$-th layer and 
$\op{Map}(G, Y)$ the $(i+1)$-st layer.

\begin{remark}
Since $f(\vartheta)\in X$, each term $\kappa(\varphi^{-1}\vartheta,f(\vartheta))$ lies in $Y$.  
In practice, $\kappa(\cdot,\cdot)$ is often nonlinear in the second argument (for example, it may include an activation function).  
This slightly generalizes the matrix multiplication in Formula~(11) of \cite{cohen2016group}.
\end{remark}

\begin{remark}
Equation~(10) in \cite{cohen2016group} treats the zeroth layer (referred to there as the first layer) separately. In that formulation, the zeroth layer map is the equivariantization of a single layer map of a standard CNN. We now explain how the zeroth layer can be viewed and treated in the same manner as the higher layers.

Let $X_0 = \op{Map}(D, \R)$ denote the zeroth layer, i.e., the space of images, where $D \subset \R^2$ (taken to be $\Z^2$ in \cite{cohen2016group}). In the reformulated version of the G-CNN, Equation~\ref{eqn: GCNN reformulated}, the zeroth layer is taken to be $\op{Map}(G, X_0)$. Given an image $x_0 \in X_0$, we associate to it a map $f_0: G \to X_0$ defined by
\[
f_0(\varphi) = \varphi(x_0),
\]
where $\varphi(x_0)$ denotes the image obtained by applying the group element $\varphi$ to $x_0$. Concretely, if $x_0$ is represented by a function $h: D \to \R$, then $\varphi(x_0) = h \circ \varphi^{-1}: D \to \R$.

In this formulation, the zeroth layer can be transformed in exactly the same way as the subsequent layers. In particular, one can apply a group convolution directly to the zeroth layer if desired.
\end{remark}

\begin{remark}
Any $G$-equivariant map 
$\op{Map}(G,X)\to \op{Map}(G,Y)$ 
is uniquely determined by a map 
$\op{Map}(G,X)\to Y$ according to Lemma~\ref{trivial lemma},
whereas a group convolution is determined by a kernel 
$G\times X\to Y$.  
Since $\op{Map}(G,X)$ is generally much larger than $G\times X$, most $G$-equivariant maps do \emph{not} arise from a group convolution.  
For example, pooling layers are not obtained from a group convolution.
\end{remark}

For completeness we verify $G$-equivariance of $\Phi$.  
We need to show
\[
\ell(\Phi f)=\Phi(\ell f)\qquad \forall\,\ell\in G.
\]
Evaluating at $\varphi\in G$, the left-hand side is
\[
(\ell(\Phi f))\varphi
=(\Phi f)(\ell^{-1}\varphi)
=\sum_{\vartheta\in G}\kappa(\varphi^{-1}\ell\vartheta,f(\vartheta)),
\]
while the right-hand side is
\[
(\Phi(\ell f))\varphi
=\sum_{\vartheta\in G}\kappa(\varphi^{-1}\vartheta,(\ell f)(\vartheta))
=\sum_{\vartheta\in G}\kappa(\varphi^{-1}\vartheta,f(\ell^{-1}\vartheta)) 
=\sum_{\xi\in G}\kappa(\varphi^{-1}\ell\xi,f(\xi)),
\]
where the last equality uses the change of variable $\xi=\ell^{-1}\vartheta$.  
Hence the two sides agree.

\medskip
\noindent\textbf{Comparison with Our Constructions.}

\smallskip
\emph{First construction.}  
Given $F:X\to Y$, we build a map $\widehat{X}\to\widehat{Y}$ as in Figure~\ref{fig:vanilla}. This is a special case of group convolution.
\begin{figure}[h]
\centering
\begin{tikzcd}
\widehat{X} \arrow[r,"\widehat{F\circ p_1}"] \arrow[dr,swap,sloped, "F\circ p_1"']  \arrow[d,swap, "p_1"]
& \widehat{Y}\arrow[d,"p_2"]\\
X\arrow[r,"F"]&Y
\end{tikzcd}
\caption{For any $F:X\to Y$, we construct a map $\widehat{X}\to\widehat{Y}$.}
\label{fig:vanilla}
\end{figure}
This layer map $\widehat{F\circ p_1}$ arises from the convolution kernel 
\[
\kappa:G\times X\to Y,\qquad 
\kappa(e,x)=F(x)\;\text{ for all }x\in X,\quad 
\kappa(\varphi,x)=0\;\text{ for }\varphi\neq e.
\]
Indeed, for $f\in\widehat{X}=\op{Map}(G,X)$,
\[
\Phi(f)(\varphi)
=\sum_{\vartheta\in G}\kappa(\varphi^{-1}\vartheta,f(\vartheta))
=\kappa(e,f(\varphi))
=F(f(\varphi))
=\widehat{F\circ p_1}(f)(\varphi).
\]

\smallskip
\emph{Second construction.}  
(We call this “layer-by-layer equivarification,” though the name may be refined.)  
By Lemma~\ref{trivial lemma}, to construct a $G$-equivariant map 
$\op{Map}(G,X)\to\op{Map}(G,Y)$ 
it suffices to define a map 
$\op{Map}(G,X)\to Y$. We know explain that a group convolution is a special case of this.

Given a convolution kernel $\kappa:G\times X\to Y$, define 
\[
\Theta:\op{Map}(G,X)\to Y,
\qquad
\Theta(f)=\sum_{\vartheta\in G}\kappa(\vartheta,f(\vartheta)),
\]
i.e.\ the “trace” of $$\kappa(\cdot,f(\cdot)):G\times G\to Y.$$
Equivarifying $\Theta$ yields
\[
\widehat{\Theta}:\op{Map}(G,X)\to\op{Map}(G,Y),
\quad
\widehat{\Theta}(f)(\varphi)
=\Theta(\varphi^{-1}f)
=\sum_{\vartheta\in G}\kappa(\vartheta,f(\varphi\vartheta))
=\sum_{\xi\in G}\kappa(\varphi^{-1}\xi,f(\xi)),
\]
where the last equality uses the change of variable $\xi=\varphi\vartheta$.  
This is exactly the group convolution.  
In short, \emph{group convolution is the equivarification of the trace associated with the convolution kernel}.

\smallskip
\emph{Pooling.}  
Finally, using the same $F:X\to Y$ as above, and suppose $Y = \R$, define 
\[
\Psi:\op{Map}(G,X)\to Y,
\qquad
\Psi(f)=\max_{\vartheta\in G}F\bigl(f(\vartheta)\bigr),
\]
which corresponds to a max-pooling layer.  
In genera,l such a map is \emph{not} a special case of a group convolution.

\section{Numerical experiments} \label{sec:experiments}

In this section, we demonstrate the training process of our equivarified networks and compare them with the data-augmented counterparts. For convenience of discussion, we focus on the global equivarification approach \eqref{eq:equinet global}, as our primary experiments have shown that the layer-wise construction \eqref{eq:equinet layer by layer} exhibits similar performance to the global version under a comparable number of parameters.\footnote{We refer the sample codes to:
\url{https://github.com/symplecticgeometry/equivariant-neural-networks-and-equivarification}}

\subsection{Label and loss function}

We explain how we encode labels when training the equivariant network under a finite symmetry group \( G \). Suppose the input data consists of \( k \) classes, and let \( X_{n+1} = \mathbb{R}^k \) denote the label space. Each data point \( x_0 \) from class \( m \in \{1,\dots,k\} \) is assigned a one-hot label
\[
y = (0, \dots, 0, \underbrace{1}_{m\text{-th position}}, 0, \dots, 0) \in \mathbb{R}^{k}.
\]
In this case, a classifier $F$ maps $X_0$ to $\R^k$. In standard data augmentation approaches, each transformed input \( g^{-1}x_0 \) is typically labeled with the same \( y = e_m \), aiming to enforce group \emph{invariance}. However, this setup fails to preserve the analytical group structure in the label space.

In contrast, our method is to construct a \emph{\( G \)-equivariant} classifier
\[
\widehat{F}: X_0 \to \widehat{X}_{n+1} = \mathbb{R}^{|G| \cdot k}
\]
by lifting the outputs across all transformed inputs:
\[
\widehat{F}(x_0) = \left( F(x_0), \, F(g_2 x_0), \, \dots, \, F(g_{|G|} x_0) \right),
\]
where \( g_1 = e \in G \) is the identity. The label is also lifted to a padded one-hot vector\footnote{When a general (possibly infinite) symmetry group $G$ is considered, our construction amounts to labeling $\hF(x)(e) = F(x)$ with $y$ and $\hF(x)(g)$, $g\neq e$, with $\mathbf{0}$. }:
\[
\widehat{y} = y \oplus 0 \oplus \cdots \oplus 0 \in \mathbb{R}^{|G| \cdot k},
\]
which can be interpreted as a \( |G| \)-tuple of vectors in \( \mathbb{R}^k \), with only the first component nonzero.

To maintain equivariance, the transformed input \( gx_0 \) is labeled by
\[
gx_0 \mapsto g \cdot \widehat{y},
\]
where the action of \( g \) on $\widehat{y}$ is implemented by tuple permutation, consistent with our network architecture (see Section~\ref{sec: Application to neural networks}).

\medskip
\begin{example}
Consider \( G = C_4 = \{e, g, g^2, g^3\} \), where \( g \) represents a 90° rotation when acting on images, and a permutation $g(a_1,a_2,a_3,a_4) = (a_2,a_3,a_4,a_1)$ when acting on $4-$tuples. Then we set up the equivariant labels for the input $x_0$ and its transformed versions,
\begin{equation}\label{eq: permuted label}
\begin{aligned}
x_0 &\mapsto \widehat{y} = y \oplus 0 \oplus 0 \oplus 0, \\
gx_0 &\mapsto g \cdot \widehat{y} = 0 \oplus 0 \oplus 0 \oplus y, \\
g^2 x_0 &\mapsto g^2 \cdot \widehat{y} = 0 \oplus 0 \oplus y \oplus 0, \\
g^3 x_0 &\mapsto g^3 \cdot \widehat{y} = 0 \oplus y \oplus 0 \oplus 0.
\end{aligned}
\end{equation}
\end{example} 

Our group-aware labeling enables the network not only to predict the input class but also to infer the applied transformation, a capability not present in standard invariant labeling. To be more specific, let $\widetilde{x_0}$ be a transformed version of $x_0$ and suppose that the largest component of $\hF(\widetilde{x_0})$ occurs in the copy $F(g\widetilde{x_0})$, then we can predict $\widetilde{x_0} = g^{-1}x_0$.

In our following experiments, we compare data augmentation and network equivarification based on a finite operation group $G$. For a $G-$augmented network, the training employs the loss function,
\[
\frac{1}{N}\sum^N_{i=1}\sum_{g\in G} l(F(g x_i),y_i),
\]
with $l$ being the cross-entropy loss and $y_i$ being the true one-hot label of $x_i$. The $G-$equivariant network is trained with the loss function,
\[
\frac{1}{N}\sum^N_{i=1}l(\hF(x_i),\widehat{y}_i),
\]
where the extended label $\widehat{y}_i$ takes the form $y_i\oplus 0\oplus\cdots\oplus 0$. Note that due to the group equivariance, we have $l(\hF(gx_i),g\cdot\widehat{y}_i) = l(g\cdot\hF(x_i),g\cdot\widehat{y}_i) = l(\hF(x_i),\widehat{y}_i)$.  The predicted label of data $x$ is given by $i_{max}\mod num\_ classes$, where $i_{max}$ is the index of the largest component of $\hF(x)$. Again, we emphasize that the equivarified network $\hF$, due to weight sharing, has the same number of parameters as the original model $F$. 

For fairness of comparison, in all test cases we apply the Adam optimizer with a learning rate $10^{-3}$ to update the network parameters. The training data are fed to the models with a batch size of 100.

\subsection{MNIST}

We compare the data augmentation approach and our network equivarification on the MNIST dataset, which contains 70,000 grayscale images of handwritten digits 0 $\sim$ 9, with 60,000 used for training and 10,000 for testing. To improve generalization and prevent overfitting, we further divide the training data into a training set of 50,000 and a validation set of 10,000. All images are normalized with the mean of $0.1307$ and the standard deviation of $0.3081$. As a baseline model, we adopt LeNet-5 with approximately 60K parameters. All networks are trained 10 epochs (5000 iterations).

It is observed from Table \ref{tab:mnist results} that both network equivarification and data augmentation yield comparable performance on the MNIST data set, effectively enhancing the generalizability of the vanilla CNN when test images undergo random transformations. While the equivariant networks better preserve the accuracy when the transformations are drawn from the same group (e.g., $C_4$ or $D_4$) used in training, the data-augmented counterparts exhibit slightly better robustness when the test images are subject to random-angle rotations beyond the group. These results indicate a practical trade-off: while group-equivariant networks offer strong performance under structured, known symmetries, data augmentation may sometimes provide better robustness under broader, less structured transformations. 

Figure~\ref{fig: check equivariance} illustrates the softmax output of our $C_4$-equivariant network for a digit “7” image under $0^\circ$, $90^\circ$, $180^\circ$, and $270^\circ$ rotations. The predicted vectors exhibit exact cyclic permutation under the group action, verifying that the learned network satisfies the $C_4$-equivariance property by construction.

\begin{table}[h!]
    \small
    \setstretch{1.2}
    \centering
    \begin{tabular}{|c||c|c|c|c|c|}
    \hline
     \diagbox{Test Acc}{Models}  & Vanilla CNN & $C_4-$Aug. & $C_4-$Equiv. &  $D_4-$Aug. & $D_4-$Equiv.\\
     \hline 
     Test set original & $98.7\%$& $98.2\%$ & $98.7\%$ & $96.3\%$ & $97.5\%$ \\
     \hline
     Random $C_4$ transform on test set & $41.5\%$ & $98.1\%$ & $98.7\%$ & $96.1\%$ & $97.5\%$ \\
     \hline
     Random $D_4$ transform on test set & $36.7\%$ & $70.7\%$& $66.7\%$& $96.2\%$ & $97.5\%$ \\
     \hline
     Random rotation ($0^\circ\sim360^\circ$) on test set & $39.6\%$ & $86.8\%$ & $85.6\%$& $83.6\%$& $82.5\%$\\
     \hline
    \end{tabular}
    \caption{MNIST test accuracy. Vanilla CNN (no augmentation) vs. group-augmented CNN vs. group-equivarified CNN. ``Aug." denotes data augmentation, ``Equiv." denotes group-equivariant networks.  Results average 10 independent evaluations.}
    \label{tab:mnist results}
\end{table}

\begin{figure}[h]
\centering
\begin{minipage}{0.2\textwidth}
    \includegraphics[width=\textwidth]{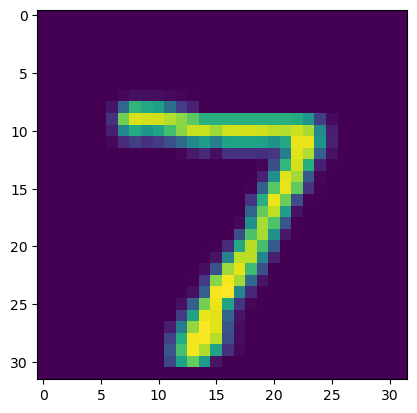}
    \subcaption{digit $7$, angle $0^\circ$}
\end{minipage}
\quad
\begin{minipage}{0.2\textwidth}
\tiny
\begin{align*}
    [&3.0238e-09,\ 4.4535e-07,\ 1.1168e-07,\ 8.1438e-07,\ 2.1626e-09, \ 2.4075e-10,\\
         &2.7370e-12, \ \boxed{9.9999e-01}, \ 6.0974e-09, \ 4.9288e-06, \ 8.9081e-14, 7.5260e-15,\\
         &2.6949e-15, \ 9.8918e-16, \ 3.8059e-16, \ 6.8213e-13, \ 4.5590e-15, \ 8.7373e-13,\\
         &4.2735e-13,\ 1.2758e-15,\ 1.8902e-14,\ 2.5661e-12,\ 1.8616e-14,\ 4.0692e-16,\\
         &2.4703e-12,\ 1.1981e-14,\ 6.5862e-11,\ 5.5811e-18,\ 9.7830e-16,\ 1.8903e-14,\\
         &2.9701e-12,\ 1.0142e-13,\ 1.5757e-10,\ 1.3272e-12,\ 2.3923e-09,\ 5.4502e-12,\\
         &9.8439e-12,\ 1.0966e-12,\ 1.0032e-12,\ 4.1531e-14]
\end{align*}
\end{minipage}
\\[1em]
\begin{minipage}{0.2\textwidth}
    \includegraphics[width=\textwidth]{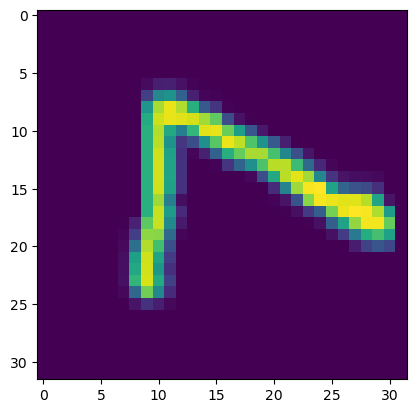}
        \subcaption{digit $7$, angle $90^\circ$}
\end{minipage}
\quad
\begin{minipage}{0.2\textwidth}
\tiny
\begin{align*}
    [&8.9081e-14,\ 7.5260e-15,\ 2.6949e-15,\ 9.8918e-16,\ 3.8059e-16, 6.8213e-13,\\
         &4.5590e-15,\ 8.7373e-13,\ 4.2735e-13,\ 1.2758e-15,\ 1.8902e-14,\ 2.5661e-12,\\
         &1.8616e-14,\ 4.0692e-16,\ 2.4703e-12,\ 1.1981e-14,\ 6.5862e-11,\ 5.5811e-18,\\
         &9.7830e-16,\ 1.8903e-14,\ 2.9701e-12,\ 1.0142e-13,\ 1.5757e-10,\ 1.3272e-12,\\
         &2.3923e-09,\ 5.4502e-12,\ 9.8439e-12,\ 1.0966e-12,\ 1.0032e-12,\ 4.1531e-14,\\
         &3.0238e-09,\ 4.4535e-07,\ 1.1168e-07,\ 8.1438e-07,\ 2.1626e-09,\ 2.4075e-10,\\
         &2.7370e-12,\ \boxed{9.9999e-01},\ 6.0974e-09,\ 4.9288e-06]
\end{align*}
\end{minipage}
\\[1em]
\begin{minipage}{0.2\textwidth}
    \includegraphics[width=\textwidth]{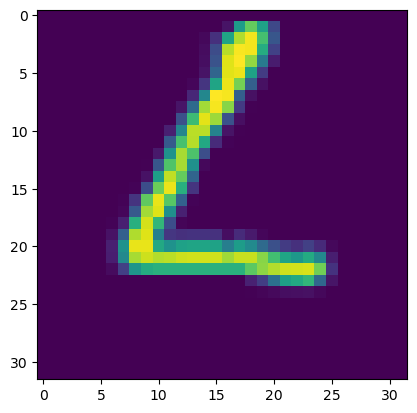}
        \subcaption{digit $7$, angle $180^\circ$}
\end{minipage}
\quad
\begin{minipage}{0.2\textwidth}
\tiny
\begin{align*}
    [&1.8902e-14,\ 2.5661e-12,\ 1.8616e-14,\ 4.0692e-16,\ 2.4703e-12,\ 1.1981e-14,\\
         &6.5862e-11,\ 5.5811e-18,\ 9.7830e-16,\ 1.8903e-14,\ 2.9701e-12,\ 1.0142e-13,\\
         &1.5757e-10,\ 1.3272e-12,\ 2.3923e-09,\ 5.4502e-12,\ 9.8439e-12,\ 1.0966e-12,\\
         &1.0032e-12,\ 4.1531e-14,\ 3.0238e-09,\ 4.4535e-07,\ 1.1168e-07,\ 8.1438e-07,\\
         &2.1626e-09,\ 2.4075e-10,\ 2.7370e-12,\ \boxed{9.9999e-01},\ 6.0974e-09,\ 4.9288e-06,\\
         &8.9081e-14,\ 7.5260e-15,\ 2.6949e-15,\ 9.8918e-16,\ 3.8059e-16,\ 6.8213e-13,\\
         &4.5590e-15,\ 8.7373e-13,\ 4.2735e-13,\ 1.2758e-15]
\end{align*}
\end{minipage}
\\[1em]
\begin{minipage}{0.2\textwidth}
    \includegraphics[width=\textwidth]{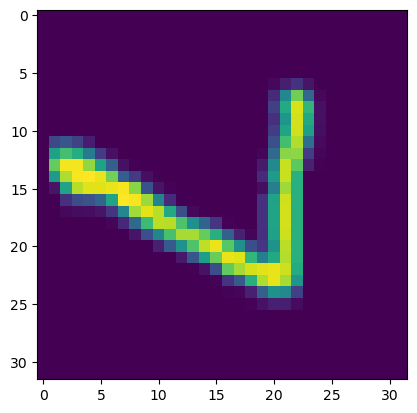}
        \subcaption{digit $7$, angle $270^\circ$}
\end{minipage}
\quad
\begin{minipage}{0.2\textwidth}
\tiny
\begin{align*}
    [&2.9701e-12,\ 1.0142e-13,\ 1.5757e-10,\ 1.3272e-12,\ 2.3923e-09,\ 5.4502e-12,\\
         &9.8439e-12,\ 1.0966e-12,\ 1.0032e-12,\ 4.1531e-14,\ 3.0238e-09,\ 4.4535e-07,\\
         &1.1168e-07,\ 8.1438e-07,\ 2.1626e-09,\ 2.4075e-10,\ 2.7370e-12, \ \boxed{9.9999e-01},\\
         &6.0974e-09,\ 4.9288e-06,\ 8.9081e-14,\ 7.5260e-15,\ 2.6949e-15,\ 9.8918e-16,\\
         &3.8059e-16,\ 6.8213e-13,\ 4.5590e-15,\ 8.7373e-13,\ 4.2735e-13,\ 1.2758e-15,\\
         &1.8902e-14,\ 2.5661e-12,\ 1.8616e-14,\ 4.0692e-16,\ 2.4703e-12,\ 1.1981e-14,\\
         &6.5862e-11,\ 5.5811e-18,\ 9.7830e-16,\ 1.8903e-14]
\end{align*}
\end{minipage}
\caption{Right: rotated images. Left: probability vectors in $\R^{40}$ obtained by applying softmax function to logit data. The largest component in each output vector is boxed, indicating the label and the transform. The results verify that when the input image is transformed, the output of our equivariant network is exactly equal to the permutation of that corresponding to the untransformed image (see \eqref{eq: permuted label}).}
\label{fig: check equivariance}

\end{figure}

We further evaluate the performance of different models when the size of the training set is reduced to $20{,}000$ images (from the full $60{,}000$ in MNIST). In this low-data regime, the benefits of group-equivariant networks become more pronounced, particularly when using the more expressive $D_4$ group. This improvement may be attributed to the multi-view strategy employed by our equivarification framework, which effectively increases the model’s exposure to structured variations without requiring additional labeled data. By explicitly modeling symmetry, equivariant networks alleviate the burden on representation learning in data-scarce settings.

\begin{table}[h!]
    \small
    \setstretch{1.2}
    \centering
    \begin{tabular}{|c||c|c|c|c|c|}
    \hline
     \diagbox{Test Acc}{Models}  & Vanilla CNN & $C_4-$Aug. & $C_4-$Equiv. &  $D_4-$Aug. & $D_4-$Equiv.\\
     \hline 
     Test set original & $98.4\%$& $97.0\%$ & $97.9\%$ & $92.5\%$ & $96.6\%$ \\
     \hline
     Random $C_4$ transform on test set & $39.4\%$ & $97.0\%$ & $97.9\%$ & $92.5\%$ & $96.6\%$ \\
     \hline
     Random $D_4$ transform on test set & $33.1\%$ & $70.3\%$& $65.8\%$& $92.4\%$ & $96.6\%$ \\
     \hline
     Random rotation ($0^\circ\sim360^\circ$) on test set & $36.2\%$ & $83.7\%$ & $81.9\%$& $77.3\%$& $79.1\%$\\
     \hline
    \end{tabular}
    \caption{MNIST test accuracy under reduced training data (20,000 out of 60,000). Vanilla CNN (no augmentation) vs. group-augmented CNN vs. group-equivarified CNN. ``Aug." denotes data augmentation, ``Equiv." denotes group-equivariant networks.}
    \label{tab:mnist results}
\end{table}

\subsection{SVHN}

Similar to the MNIST task, the SVHN classification problem also targets digit recognition from 0 to 9. However, SVHN poses a greater challenge due to its more complex visual features, including real-world image noise, color variability, and background clutter. In our experiments, we use 63,257 training images, 10,000 for validation, and 10,000 for testing. All inputs are normalized using channel-wise statistics: mean $(0.4377, 0.4438, 0.4728)$ and standard deviation $(0.1980, 0.2010, 0.1970)$. We train all models using a batch size of 100 and 20 epochs (12,660 iterations).

In our first test, we continue to use LeNet5 as the baseline architecture, modifying its input layer to accommodate three color channels. As shown in Table~\ref{tab:svhn lenet results}, the vanilla CNN, trained without augmentation, performs well on the original test set, where most digits are upright and centered. However, it suffers a significant drop in accuracy when evaluated on transformed inputs, indicating a strong reliance on the orientation bias in the training data. In contrast, both group-augmented and group-equivariant models exhibit significantly better robustness under $C_4$ and $D_4$ transformations. This confirms the benefit of explicitly modeling geometric invariance for real-world image recognition tasks.

To further evaluate the scalability of our approach, we replace the backbone with a more expressive Mini-VGG architecture (approximately 200K parameters). Results are shown in Table \ref{tab:svhn Mini-VGG results}.  All models demonstrate improved accuracy compared to their LeNet5 counterparts. Notably, our equivarified networks consistently outperform data-augmented models under both $C_4$ and $D_4$ transformations. In particular, the $D_4$-equivariant model achieves $85.7\%$ on transformed test sets, compared to $80.7\%$ for LeNet5 with $D_4$-augmentation and $81.9\%$ for Mini-VGG with $D_4$-augmentation. These results demonstrate that our equivariant design scales well with larger architectures and more challenging image distributions.

\begin{table}[h!]
    \small
    \setstretch{1.2}
    \centering
    \begin{tabular}{|c||c|c|c|c|c|}
    \hline
     \diagbox{Test Acc}{Models}  & Vanilla CNN & $C_4-$Aug. & $C_4-$Equiv. &  $D_4-$Aug. & $D_4-$Equiv.\\
     \hline 
     Test set original & $88.8\%$& $81.1\%$ & $82.5\%$ & $77.4\%$ & $80.7\%$\\
     \hline
     Random $C_4$ transform on test set& $38.1\%$ & $80.9\%$ & $82.5\%$ & $76.7\%$ & $80.7\%$ \\
     \hline
     Random $D_4$ transform on test set & $31.4\%$ & $60.0\%$ & $59.5\%$ & $76.7\%$  & $80.7\%$\\
     \hline
     Random rotation ($0^\circ\sim360^\circ$) on test set & $28.1\%$ & $54.7\%$ & $53.5\%$ & $51.8\%$ & $51.1\%$\\
     \hline
    \end{tabular}
    \caption{SVHN test accuracy with LeNet5 baseline ($\approx$ 60K parameters), 20 epochs of training. Vanilla CNN (no augmentation) vs. group-augmented CNN vs. group-equivarified CNN. ``Aug." denotes data augmentation, ``Equiv." denotes group-equivariant networks.}
    \label{tab:svhn lenet results}
\end{table}

\begin{table}[h!]
    \small
    \setstretch{1.2}
    \centering
    \begin{tabular}{|c||c|c|c|c|c|}
    \hline
     \diagbox{Test Acc}{Models}  & Vanilla CNN & $C_4-$Aug. & $C_4-$Equiv. &  $D_4-$Aug. & $D_4-$Equiv.\\
     \hline 
     Test set original & $92.4\%$& $85.8\%$ & $87.1\%$ & $81.9\%$& $85.7\%$\\
     \hline
     Random $C_4$ transform on test set& $38.2\%$& $84.9\%$& $87.1\%$ & $81.6\%$&$85.7\%$\\
     \hline
     Random $D_4$ transform on test set &$32.9\%$ & $63.7\%$ & $63.1\%$& $81.4\%$ &$85.7\%$\\
     \hline
     Random rotation ($0^\circ\sim360^\circ$) on test set & $29.4\%$ & $59.7\%$& $56.7\%$ & $57.1\%$ & $54.8\%$\\
     \hline
    \end{tabular}
    \caption{SVHN test accuracy with Mini-VGG baseline ($\approx$ 200K parameters), 20 epochs of training. Vanilla CNN (no augmentation) vs. group-augmented CNN vs. group-equivarified CNN. ``Aug." denotes data augmentation, ``Equiv." denotes group-equivariant networks.}
    \label{tab:svhn Mini-VGG results}
\end{table}

\subsection{CIFAR10}

The CIFAR10 data set consists of colored images from 10 classes of objects: ``airplane'', ``automobile'', ``bird'', ``cat'', ``deer'', ``dog'', ``frog'', ``horse'', ``ship'', and ``truck''.
 Unlike MNIST or SVHN, achieving high test accuracy ($\sim$80--90\%) on CIFAR-10 generally requires deeper architectures, e.g. ResNet18 ($\approx$11.7M parameters) or VGG11 ($\approx$132M parameters), longer training schedules, and stronger regularization due to the greater visual complexity and intra-class variability. In our experiments, the input images are normalized with channel-wise mean $(0.4914, 0.4822, 0.4465)$ and standard deviation $(0.2471, 0.2435, 0.2616)$. We distribute 45,000, 5,000, and 10,000 images for training, validation, and test sets.

To highlight the comparison of different approaches, we begin with the relatively lightweight Mini-VGG baseline, which reaches reasonable convergence within 20 epochs. As shown in Table~\ref{tab:cifar10 Mini-VGG results}, both data augmentation and network equivarification significantly outperform vanilla CNN, consistent with the trends observed in the MNIST and SVHN experiments. In particular, the $C_4$-equivariant network achieves higher accuracy than its $C_4$-augmented counterpart on this more complex dataset.

Interestingly, models enhanced with $C_4$ transformations—either via data augmentation or equivariant design—perform comparably to those leveraging the larger $D_4$ group. This is not entirely surprising: the CIFAR-10 dataset already contains substantial variation in pose and orientation, such as left- and right-facing animals, making horizontal flips (a key component of $D_4$) partially redundant. Thus, expanding the symmetry group introduces additional inductive bias, but may not yield proportional gains when the data distribution already exhibits approximate invariance to those transformations. These results suggest that moderate transformation groups, such as $C_4$ in our case, may offer a favorable trade-off between model expressiveness, training efficiency, and inductive bias alignment with the data.

In Table \ref{tab:cifar10 resnet18 results}, we compare different group-reinforced approaches under the larger ResNet18 baseline, still trained with 20 epochs. In our experiments, it is observed that when applying our equivarified construction to deep architectures, using batch norms can lead to suboptimal performance due to their noisy calculations. To alleviate this, we replace batch norms with group norms, which normalizes over fixed groups of channels and is independent of batch size. Here we employ 32 channel groups for the layers except the first. This leads to more stable optimization and consistent accuracy gains in the equivariant ResNet18 models, as reflected in Table~\ref{tab:cifar10 resnet18 results}. The results confirm the improved accuracy of the equivariant networks constructed using heavier models.

\begin{table}[h!]
    \small
    \setstretch{1.2}
    \centering
     \begin{tabular}{|c||c|c|c|c|c|}
    \hline
     \diagbox{Test Acc}{Models}  & Vanilla CNN & $C_4-$Aug. & $C_4-$Equiv. &  $D_4-$Aug. & $D_4-$Equiv.\\
     \hline 
     Test set original & $77.2\%$ & $71.5\%$  & $74.5\%$ &$74.3\%$  & $74.6\%$ \\
     \hline
     Random $C_4$ transform on test set & $42.3\%$& $71.7\%$ & $74.5\%$  & $74.3\%$ & $74.6\%$   \\
     \hline
     Random $D_4$ transform on test set& $41.9\%$& $71.4\%$ & $74.1\%$ &   $73.9\%$ & $74.6\%$  \\
     \hline
     Random rotation ($0^\circ\sim360^\circ$) on test set& $32.7\%$& $47.1\%$  &  $47.2\%$& $44.6\%$  & $48.4\%$ \\
     \hline
    \end{tabular}
    \caption{CIFAR10 test accuracy with Mini-VGG baseline, 20 epochs of training. Vanilla CNN (no augmentation) vs. group-augmented CNN vs. group-equivarified CNN. ``Aug." denotes data augmentation, ``Equiv." denotes group-equivariant networks.}
    \label{tab:cifar10 Mini-VGG results}
\end{table}

\begin{table}[h!]
    \small
    \setstretch{1.2}
    \centering
    \begin{tabular}{|c||c|c|c|c|c|}
    \hline
     \diagbox{Test Acc}{Models}  & Vanilla CNN & $C_4-$Aug. & $C_4-$Equiv. &  $D_4-$Aug. & $D_4-$Equiv.\\
     \hline 
     Test set original & $81.0\%$ & $78.5\%$ & $78.8\%$ & $79.0\%$ & $79.3\%$  \\
     \hline
     Random $C_4$ transform on test set & $41.2\%$ & $78.2\%$  & $78.8\%$& $78.8\%$ & $79.3\%$   \\
     \hline
     Random $D_4$ transform on test set& $41.1\%$ & $77.7\%$ & $78.6\%$& $78.9\%$& $79.3\%$  \\
     \hline
     Random rotation ($0^\circ\sim360^\circ$) on test set& $30.4\%$ & $53.5\%$ & $49.5\%$& $52.8\%$& $51.0\%$ \\
     \hline
    \end{tabular}
    \caption{CIFAR10 test accuracy with ResNet18 baseline, 20 epochs of training. Vanilla CNN (no augmentation) vs. group-augmented CNN vs. group-equivarified CNN. ``Aug." denotes data augmentation, ``Equiv." denotes group-equivariant networks.}
    \label{tab:cifar10 resnet18 results}
\end{table}

\subsection{Group refinement and equivariance error}

In the preceding MNIST experiments, the $C_4$ (and $D_4$) group transformations emphasize equivariance under $90^\circ \cdot n$ ($n \in \mathbb{N}$) rotations. Here, we further investigate equivariant networks constructed using the $C_6$ and $C_8$ groups, which correspond to $60^\circ \cdot n$ and $45^\circ \cdot n$ rotations, respectively.

It is important to note that, in practice, $C_6$ and $C_8$ operations do not exactly preserve group structure due to numerical interpolation errors associated with non-orthogonal rotations. In contrast, $90^\circ$ rotations and horizontal/vertical flips in $C_4$ and $D_4$ can be implemented exactly on discrete image grids, thus avoiding interpolation. As a result, our multiview-based $C_6$- and $C_8$-equivarified constructions cannot maintain strict numerical equivariance.

To quantify this effect, we evaluate the relative equivariance error defined by
\[
\frac{1}{N|G|}\sum^{N}_{i=1}\sum_{g\in G}\frac{||\hF(gx)-g\hF(x)||^2}{||\hF(x)||^2},
\]
where $\hf$ denotes the network output and $g$ is a group action. Figure~\ref{fig:mnist equi err} shows the measured equivariance errors across different group constructions.

Due to the non-vanishing equivariance bias, the $C_6$- and $C_8$-equivarified models exhibit a slight drop in accuracy when tested under random group operations. However, since the relative equivariance error remains low, both models maintain strong test performance under their respective group transformations, as shown in Table~\ref{tab:mnist C468 results}.

These results highlight a key observation: numerical equivariance, even when imperfect, can still enhance model robustness, provided the equivariance error is controlled. This suggests that the success of equivariant design is not solely contingent on strict algebraic structure, but also on the approximate preservation of symmetry in practice.

\begin{figure}[h!]
    \centering
    \includegraphics[width=0.35\linewidth]{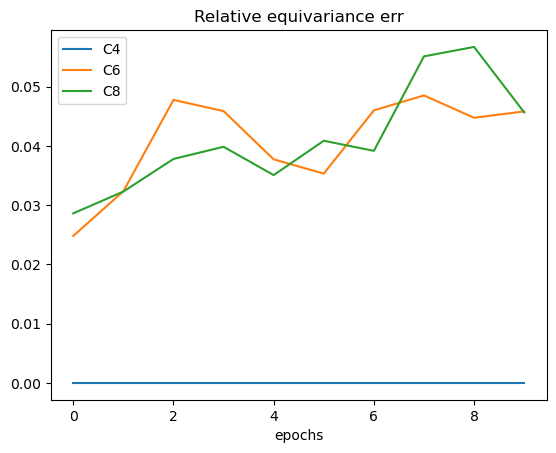}
    \caption{Relative equivariance errors on MNIST data set with respect to $C_4$, $C_6$, and $C_8$ operation groups.}
    \label{fig:mnist equi err}
\end{figure}

\begin{table}[h!]
    \small
    \setstretch{1.2}
    \centering
     \begin{tabular}{|c||c|c|c|}
    \hline
     \diagbox{Test Acc}{Models}   & $C_4-$Equiv. &  $C_6-$Equiv.  & $C_8-$Equiv. \\
     \hline 
     Test set original & $98.7\%$& $98.8\%$ & $98.7\%$ \\
     \hline
     Random $C_4$ transform on test set & $98.7\%$& $71.3\%$ & $98.7\%$ \\
     \hline
     Random $C_6$ transform on test set & $85.8\%$ & $96.3\%$ & $93.6\%$ \\
     \hline
     Random $C_8$ transform on test set & $82.4\%$ & $79.5\%$& $97.4\%$\\
     \hline
     Random rotation ($0^\circ\sim360^\circ$) on test set& $85.6\%$& $84.0\%$& $94.5\%$\\
     \hline
    \end{tabular}
    \caption{MNIST test accuracy. Comparison of equivariant models constructed using $C_4$, $C_6$, and $C_8$ rotation groups.}
    \label{tab:mnist C468 results}
\end{table}

\section{Conclusion}
We introduced a general equivarification method for constructing equivariant neural networks capable of effectively handling data with symmetries. Our approach is broadly applicable to arbitrary neural architectures and functions, utilizing group actions to uniformly define equivariant structure across layers—without requiring knowledge of the internal workings of each layer.

As a concrete example, we applied our method to a convolutional neural network for image classification. The resulting network achieves strong accuracy, while reducing both design and training complexity.

\appendix

\section{More on Equivarification}
In Section~\ref{sec: equivarification}, we defined $(\map(G, Z), p)$ as an example of a $G$-equivarification. Here, we show that it is “minimal” in the sense of satisfying a universal property.
\begin{lemma}[Universal Property]
Let $(\widehat Z', p')$ be any $G$-equivarification of $Z$. Then there exists a $G$-equivariant map $\pi: \widehat Z' \to \map(G, Z)$ such that $p' = p \circ \pi$.
Moreover, for any set $X$ and function $F: X \to Z$, the corresponding lifts $\widehat F: X \to \map(G, Z)$ and $\widehat F': X \to \widehat Z'$ satisfy
\[
\pi \circ \widehat F' = \widehat F.
\]
(See Figure~\ref{fig: factor thru}.)
\end{lemma}

\begin{proof}
Define $\pi: \widehat Z' \to \map(G, Z)$ by
\[
[\pi(\widehat z')](g) := p'(g^{-1} \cdot \widehat z'), \quad \text{for all } g \in G.
\]
Then,
\[
p \circ \pi(\widehat z') = [\pi(\widehat z')](e) = p'(\widehat z').
\]
To check $G$-equivariance, let $h \in G$:
\[
[\pi(h \cdot \widehat z')](g) = p'(g^{-1} h \cdot \widehat z'), \quad [h \cdot \pi(\widehat z')](g) = [\pi(\widehat z')](h^{-1} g) = p'(g^{-1} h \cdot \widehat z').
\]
So $\pi$ is $G$-equivariant. Finally, since both $\pi \circ \widehat F'$ and $\widehat F$ are $G$-equivariant lifts of $F$, uniqueness (Lemma~\ref{trivial lemma}) implies $\pi \circ \widehat F' = \widehat F$.
\end{proof}

% We now explore a complementary direction: minimizing the equivarification space by reducing the group size, based on how $G$ acts on $X$. Define
% \[
% N = \{g \in G \mid g \cdot x = x \text{ for all } x \in X\},
% \]
% the kernel of the action. Clearly, $N$ is a normal subgroup. We say $G$ acts \emph{effectively} on $X$ if $N = \{e\}$. When the action is not effective, we may consider the quotient group $G/N$ and the smaller equivarification space $\map(G/N, Z)$.

% Given any $F: X \to Z$, we can define a $G/N$-equivariant lift $\widehat F: X \to \map(G/N, Z)$ using the same construction as in Lemma~\ref{trivial lemma}. Since $G$ maps naturally to $G/N$, this lift remains $G$-equivariant as well.

\begin{figure}
\centering
\begin{tikzcd}
 & & \widehat Z' 
 \arrow[ddl, bend left, "p'"]
\\
{} & \widehat Z \arrow[d, "p"] 
\arrow[ur,  "{\pi}", dashed]
&  \\
 X 
 \arrow[r, "F"] 
  \arrow[ur, "\widehat F"] 
 \arrow[uurr, bend left, "\widehat F'" ] 
 & Z &
\end{tikzcd}
\caption{Factorization through the universal equivarification.}
\label{fig: factor thru}
\end{figure}

\bibliographystyle{plain}
\bibliography{Ref}

@inproceedings{cohen2016group,
  title={Group equivariant convolutional networks},
  author={Cohen, Taco and Welling, Max},
  booktitle={International conference on machine learning},
  pages={2990--2999},
  year={2016},
  organization={PMLR}
}

@inproceedings{cohen2018spherical,
  title={Spherical CNNs},
  author={Cohen, Taco S and Geiger, Mario and K{\"o}hler, Jonas and Welling, Max},
  booktitle={International Conference on Learning Representations},
  year={2018}
}

@inproceedings{cohen2017steerable,
  title={Steerable CNNs},
  author={Cohen, Taco S and Welling, Max},
  booktitle={International Conference on Learning Representations},
  year={2017}
}

@article{weiler20183d,
  title={3d steerable cnns: Learning rotationally equivariant features in volumetric data},
  author={Weiler, Maurice and Geiger, Mario and Welling, Max and Boomsma, Wouter and Cohen, Taco S},
  journal={Advances in Neural information processing systems},
  volume={31},
  year={2018}
}

@article{cohen2019general,
  title={A general theory of equivariant cnns on homogeneous spaces},
  author={Cohen, Taco S and Geiger, Mario and Weiler, Maurice},
  journal={Advances in neural information processing systems},
  volume={32},
  year={2019}
}

@article{esteves2020theoretical,
  title={Theoretical aspects of group equivariant neural networks},
  author={Esteves, Carlos},
  journal={arXiv preprint arXiv:2004.05154},
  year={2020}
}

@article{lim2023equivariant,
  title={What is... an Equivariant Neural Network?},
  author={Lim, Lek-Heng and Nelson, Bradley J},
  journal={Notices of the American Mathematical Society},
  volume={70},
  number={04},
  year={2023},
  publisher={AMS}
}

@article{gibson2024equivariant,
  title={Equivariant neural networks and piecewise linear representation theory},
  author={Gibson, Joel and Tubbenhauer, Daniel and Williamson, Geordie},
  journal={arXiv preprint arXiv:2408.00949},
  year={2024}
}

@inproceedings{bernander2021replacing,
  title={Replacing data augmentation with rotation-equivariant CNNs in image-based classification of oral cancer},
  author={Bernander, Karl Bengtsson and Lindblad, Joakim and Strand, Robin and Nystr{\"o}m, Ingela},
  booktitle={Iberoamerican Congress on Pattern Recognition},
  pages={24--33},
  year={2021},
  organization={Springer}
}

@inproceedings{bernander2022rotation,
  title={Rotation-equivariant semantic instance segmentation on biomedical images},
  author={Bernander, Karl Bengtsson and Lindblad, Joakim and Strand, Robin and Nystr{\"o}m, Ingela},
  booktitle={Annual conference on medical image understanding and analysis},
  pages={283--297},
  year={2022},
  organization={Springer}
}

@article{pandya20232,
  title={E (2) equivariant neural networks for robust galaxy morphology classification},
  author={Pandya, Sneh and Patel, Purvik and Blazek, Jonathan and others},
  journal={arXiv preprint arXiv:2311.01500},
  year={2023}
}

@article{qiao2023scale,
  title={Scale-rotation-equivariant Lie group convolution neural networks (Lie group-CNNs)},
  author={Qiao, Wei-Dong and Xu, Yang and Li, Hui},
  journal={arXiv preprint arXiv:2306.06934},
  year={2023}
}

@inproceedings{bogatskiy2020lorentz,
  title={Lorentz group equivariant neural network for particle physics},
  author={Bogatskiy, Alexander and Anderson, Brandon and Offermann, Jan and Roussi, Marwah and Miller, David and Kondor, Risi},
  booktitle={International Conference on Machine Learning},
  pages={992--1002},
  year={2020},
  organization={PMLR}
}

@article{batzner20223,
  title={E (3)-equivariant graph neural networks for data-efficient and accurate interatomic potentials},
  author={Batzner, Simon and Musaelian, Albert and Sun, Lixin and Geiger, Mario and Mailoa, Jonathan P and Kornbluth, Mordechai and Molinari, Nicola and Smidt, Tess E and Kozinsky, Boris},
  journal={Nature communications},
  volume={13},
  number={1},
  pages={2453},
  year={2022},
  publisher={Nature Publishing Group UK London}
}

@article{favoni2022lattice,
  title={Lattice gauge equivariant convolutional neural networks},
  author={Favoni, Matteo and Ipp, Andreas and M{\"u}ller, David I and Schuh, Daniel},
  journal={Physical Review Letters},
  volume={128},
  number={3},
  pages={032003},
  year={2022},
  publisher={APS}
}

@article{yuan2024equivariant,
  title={Equivariant neural network force fields for magnetic materials},
  author={Yuan, Zilong and Xu, Zhiming and Li, He and Cheng, Xinle and Tao, Honggeng and Tang, Zechen and Zhou, Zhiyuan and Duan, Wenhui and Xu, Yong},
  journal={Quantum Frontiers},
  volume={3},
  number={1},
  pages={8},
  year={2024},
  publisher={Springer}
}

@article{gong2023general,
  title={General framework for E (3)-equivariant neural network representation of density functional theory Hamiltonian},
  author={Gong, Xiaoxun and Li, He and Zou, Nianlong and Xu, Runzhang and Duan, Wenhui and Xu, Yong},
  journal={Nature Communications},
  volume={14},
  number={1},
  pages={2848},
  year={2023},
  publisher={Nature Publishing Group UK London}
}

@article{tang2024deep,
  title={A deep equivariant neural network approach for efficient hybrid density functional calculations},
  author={Tang, Zechen and Li, He and Lin, Peize and Gong, Xiaoxun and Jin, Gan and He, Lixin and Jiang, Hong and Ren, Xinguo and Duan, Wenhui and Xu, Yong},
  journal={Nature Communications},
  volume={15},
  number={1},
  pages={8815},
  year={2024},
  publisher={Nature Publishing Group UK London}
}

@article{van2023equireact,
  title={EquiReact: an equivariant neural network for chemical reactions},
  author={van Gerwen, Puck and Briling, Ksenia R and Bunne, Charlotte and Somnath, Vignesh Ram and Laplaza, Ruben and Krause, Andreas and Corminboeuf, Clemence},
  journal={Preprint at https://arxiv. org/abs/2312.08307 v2},
  year={2023}
}

@article{cremer2023equivariant,
  title={Equivariant graph neural networks for toxicity prediction},
  author={Cremer, Julian and Medrano Sandonas, Leonardo and Tkatchenko, Alexandre and Clevert, Djork-Arn{\'e} and De Fabritiis, Gianni},
  journal={Chemical Research in Toxicology},
  volume={36},
  number={10},
  pages={1561--1573},
  year={2023},
  publisher={ACS Publications}
}

@article{seo2025se,
  title={SE (3)-equivariant Robot Learning and Control: A Tutorial Survey},
  author={Seo, Joohwan and Yoo, Soochul and Chang, Junwoo and An, Hyunseok and Ryu, Hyunwoo and Lee, Soomi and Kruthiventy, Arvind and Choi, Jongeun and Horowitz, Roberto},
  journal={International Journal of Control, Automation and Systems},
  volume={23},
  number={5},
  pages={1271--1306},
  year={2025},
  publisher={Springer}
}

@inproceedings{hutchinson2021lietransformer,
  title={Lietransformer: Equivariant self-attention for lie groups},
  author={Hutchinson, Michael J and Le Lan, Charline and Zaidi, Sheheryar and Dupont, Emilien and Teh, Yee Whye and Kim, Hyunjik},
  booktitle={International conference on machine learning},
  pages={4533--4543},
  year={2021},
  organization={PMLR}
}

@article{white2022equivariant,
  title={Equivariant transduction through invariant alignment},
  author={White, Jennifer C and Cotterell, Ryan},
  journal={arXiv preprint arXiv:2209.10926},
  year={2022}
}

@article{mondal2023equivariant,
  title={Equivariant adaptation of large pretrained models},
  author={Mondal, Arnab Kumar and Panigrahi, Siba Smarak and Kaba, Oumar and Mudumba, Sai Rajeswar and Ravanbakhsh, Siamak},
  journal={Advances in Neural Information Processing Systems},
  volume={36},
  pages={50293--50309},
  year={2023}
}

@inproceedings{basu2023equi,
  title={Equi-tuning: Group equivariant fine-tuning of pretrained models},
  author={Basu, Sourya and Sattigeri, Prasanna and Ramamurthy, Karthikeyan Natesan and Chenthamarakshan, Vijil and Varshney, Kush R and Varshney, Lav R and Das, Payel},
  booktitle={Proceedings of the AAAI Conference on Artificial Intelligence},
  volume={37},
  number={6},
  pages={6788--6796},
  year={2023}
}

\end{document}